\documentclass[oribibl]{llncs}
\usepackage{graphicx}
\usepackage[T1]{fontenc}
\usepackage{booktabs}
\usepackage{subcaption}
\usepackage{algorithm}
\usepackage{algpseudocode}
\usepackage{amsmath, amssymb}
\usepackage{orcidlink}
\usepackage[fontsize=12pt]{fontsize}
\usepackage[square,numbers]{natbib}

\begin{document}

\title{Overcoming Challenges of Interpretive Structural Modeling with Large Language Models}

\author{Everett Rush \inst{1,2} \orcidlink{0000-0002-5632-5723} \and
David J. Icove \inst{1} \orcidlink{0000-0001-7403-9299} \and
Ari Kim \inst{2} \orcidlink{0009-0004-2994-3261} \and
Byung H. Park \inst{2} \orcidlink{0000-0002-2387-0880} \and \\
Michael A. Langston \inst{1} \orcidlink{0000-0001-5945-5796} }

\institute{$^1$Department of Electrical Engineering and Computer Science, \\ University of Tennessee, Knoxville, TN \\ \email{ \{langston, icove\}@tennessee.edu} \\ $^2$ Oak Ridge National Laboratory, Oak Ridge, TN \\ \email{\{rushen, parkbh, kima\}@ornl.gov} }

\maketitle

\abstract{
Interpretive Structural Modeling (ISM) is a well-known process for multi-criteria decision making. The success of ISM over other methodologies is its ability to model causal relationships, the binary scale of factors, and resulting hierarchical representation. Traditionally, the modeling process is performed by repeated interactions with subject matter experts until consensus is reached. This process is tedious, labor-intense, and most importantly limits the ability of ISM to scale to studies with hundreds of variables. Drawing on existing work of causal graph discovery with large language models (LLM) as imperfect experts, this work explores an integrated LLM-ISM approach for ISM. Pairwise, k-wise, rowwise, and full graph discovery methodologies are compared and evaluated. It is shown that causal graph discovery methods for ISM perform best using rowwise (SHD=160, F1-score=0.77) and full graph methods (SHD=135, F1-score=0.73). 

}

\section{Introduction}

Interpretive Structural modeling (ISM) is an important multi-criteria decision making technique developed by Warfield \cite{ismWarfield}. ISM has been successfully applied to thousands of problems in supply chain, logistics, sustainability, business research, and engineering \cite{ismSurvey}. The ISM process iteratively organizes mental models into hierarchical models useful for theory building. The success of ISM over other methodologies is the ability to model causal relationships, the binary scale of factors, and the resulting hierarchical representation. 
\par \hspace{2em}
Traditionally, the modeling step in the ISM process is performed by repeated interactions with subject matter experts (SME) until consensus is reached. This process is labor intensive, error-prone, and limits the ability of ISM to scale beyond a small number of factors. Most ISM studies range from 5 to 32 factors where SMEs are responsible for annotating all pairwise interactions between factors \cite{smartISM}. It is also common for an SME to perform multiple rounds of annotations to resolve biases \cite{sushil2018}. 
\par \hspace{2em}
The ISM process is appropriate for modeling activities when there is minimal knowledge or data on the pairwise connections between factors \cite{ismSurvey}. Consequently, the process relies heavily on SMEs rather than statistical causal inference and numerical observations. Recent studies have shown great success incorporating large language models (LLMs) with causal inference frameworks \cite{causcientist, BanTaiyu2023FQTt, vashishtha2025causal}. The general knowledge from LLMs is used to guide causal algorithms. Despite these advances, there is no comprehensive study on the use of LLMs in the ISM process.
\par \hspace{2em}
Recent studies estimate that between 47\% and 56\% of all worker tasks can be completed significantly faster and at the same level of quality given the proper tooling \cite{EloundouTyna2023GaGA}. Based on the success of LLMs in causal graph discovery and rich general knowledge that LLMs bring, this paper describes an integrated approach for the use of LLMs in the ISM process. This tool is made publicly available on Github ( https://github.com/rusheniii/LLM-InterpretiveStructuralModeling ). 

\section{Background}

While LLMs have been used extensively in causal graph discovery, LLMs are just now starting to be used for ISM \cite{bertRCNN}. This section describes the ISM process and LLM methods for causal graph discovery. The entire ISM process is presented followed by current methods in causal graph discovery with LLMs.

\subsection{Interpretive Structural Modeling}
The theory development process requires answering what factors are involved in the phenomenon of interest, how the factors are related, and why the relation exists \cite{theoryContributions}. Often causality is a key component of theory building. The ISM is a process that answers the "what" and "how" of theory building.
 \par \hspace{2em}
The ISM framework has three phases: preparation, operation, and post-operation. ISM starts with identification of a problem and definition of key factors to consider. Then a contextual relation is chosen. Next, the structural self-interaction matrix (SSIM) is formed by pairwise comparison of one factor against another. The transitive closure of the SSIM is computed, the factors are partitioned into levels, and the final model is extracted. Formally, we define the transitive closure $A^{*}$ of the adjacency matrix $A$ as 
\begin{equation}
    A^{*}  = (A+I)^{n+1} = (A+I)^n \neq (A+I)^{n-1}
\end{equation} using Boolean algebra operations. We use an implementation of the Floyd-Warshall algorithm to find the transitive closure rather than matrix operations \cite{fw_algorithm}.
\par \hspace{2em}
Traditionally, the SSIM is constructed using literature reviews, SMEs, or surveys. This step is normally treated as a data collection process that can be costly and time consuming especially when SMEs or surveys are involved. Recent methods have sought to use BERT with convolutional neural networks (Bert-RCNN) to create the SSIM matrix \cite{bertRCNN}. The Bert-RCNN model is not included in this study as the model is not publicly available and no comparisons were made to other methods.

\subsection{Causal Graph Discovery}

The use of LLMS in causal graph discovery has attracted attention of researchers in recent years \cite{ancestralConstraints,vashishtha2025causal, BanTaiyu2023FQTt}. In casual graphs, variables are represented as the node set, and causal relationships between variables  represent the edge set. Some authors enforce an acyclic constraint on the directed graph and define casual graphs as directed acyclic graphs. However, the more general definition allows for interaction between variables. The ISM methodology uses the more general definition of causal graphs and allows for cycles. LLMs have been shown to be effective in discovering the causal graph for a specified set of variables.

\begin{definition}
A \textbf{causal graph} $G$ is a directed graph consisting of variables $X=\{X_1,...,X_n\}$ and a set of edges $E$. A directed edge $X_i \rightarrow X_j \in E$ denotes a direct causal influence of $X_i$ on $X_j$ with respect to $X_1,...,X_n$. Let the \textbf{descendants} of $X_i$ be $de(X_i)=\{X_k | X_i \rightarrow ... \rightarrow X_k\}$. Let the \textbf{children} of $X_i$ be $ch(X_i) = \{X_k| X_i\rightarrow X_k\}$.
\end{definition}
Note that this definition allows for cycles. For example, consider the cyclic relationship between increasing temperature and the amount of sea ice 
\cite{TubingerPairs}.
\par \hspace{2em}
There are a number of proposed approaches for discovering causal graphs using LLMs. There is the full graph discovery method that attempts to discover the entire graph in a single LLM inference step. There is rowwise prompting that uses a single prompt that asks the LLM to determine $n-1$ number of relationships against one variable in a single prompt. There is also k-wise prompting that determines $k$ relationships against one variable in a single prompt. Finally, there is pairwise prompting that asks the LLM to determine the relationship between a single pair of variables at a time. These methods are depicted in Figure \ref{fig:tp_ism} for a 12x12 adjacency matrix.

\begin{figure}[h!]
\centering
  \includegraphics[scale=0.5]{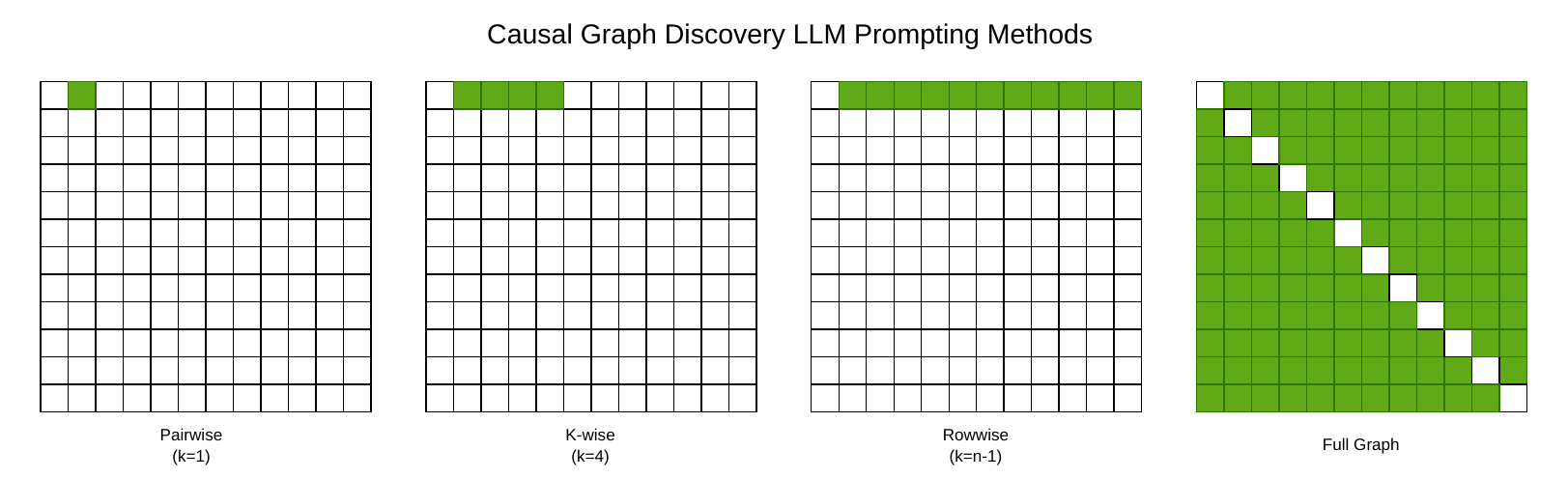}
  \caption{Prompting methods for causal graph discovery in a single prompt. Pairwise (k=1) prompting considers only two variables at once (left). As $k$ increases, the number of variables compared in a single prompt increases up to the entire row (k=n-1). Full graph prompting tries to discover the entire graph in a single prompt (right). }
  \label{fig:tp_ism}
\end{figure}

\par \hspace{2em}
One challenge in causal graph discovery is managing indirect causal relationships. For example consider causal graph $G(X, E)$, factor $x \in X$ may influence another factor $y \in X$ via an intermediate variable $z \in X$. Then $(x,z) \in E$, $(z,y) \in E$ and $(x,y) \notin E$. If the variable $z \notin X$, then the edge $(x,y)$ would be in the edge set. It is useful to define a perfect expert that has access to the true causal graph and can perfectly answer pairwise queries provided an auxiliary set of variables.

\begin{definition}
Let $G(X,E)$ be a causal graph. Given $X_i$, $X_j$, and a set of \textbf{auxiliary} nodes $O_{ij} \subseteq X$, a \textbf{perfect expert} has access to graph $G$ and responds with true to "does $X_i$ cause $X_j$" if $X_j \in ch(X_i)$ or if a direct path exists from $X_i$ to $X_j$ such that it does not contain any node in $O_{ij}$.
\label{pe}
\end{definition}

In Vashitha, et al., it is shown that the graph discovered by questioning a perfect expert can fail to find the true causal graph if the set of auxiliary nodes is empty \cite{vashishtha2025causal}. Stating this fact formally:
\begin{proposition}
 Let $G$ be the true causal graph and $\hat{G}$ be the causal graph obtained from a perfect expert when $O_{ij}=\emptyset$ $\forall i,j$ . The structural hamming distance (SHD) is shown to be $SHD(G,\hat{G})= \sum_{i=1}^{|X|} |de(X_i)| -|ch(X_i)|$. 
 \label{pe_shd}
\end{proposition}

\section{Methods}

Our method diverges from existing techniques by (1) using a more general definition of causal graphs that includes cycles,
 (2) discovers $SSIM^*$ rather than the $SSIM$ directly, and optionally (3) proposes an extension using retrieval augmented generation (RAG) to provide in-context learning to the LLM.
\par \hspace{2em}
Our approach directly discovers the transitive closure of the SSIM because all the analysis steps are done on the transitive closure of the SSIM. SSIM is just an intermediate artifact in the ISM process. Additionally, Proposition \ref{pe_shd} shows that there is an inherent penalty on graph discovery based on auxiliary factors when using SSIM directly. The transitive closure removes this penalty and allows for a better benchmark of the resulting structural models as a perfect expert can discover the transitive closure with any subset of variables. This is expressed formally in Proposition \ref{shd_0}. 
 
\begin{proposition}
Let $G^*$ be the transitive closure of the causal graph $G$. Let $\hat{G}$ be the causal graph generated by a perfect expert when $\mathbf{O_{ij}}\subseteq \mathbf{X}\setminus\{X_i,X_j\}$ $\forall i,j$. Then $G^{*}=\hat{G}^{*}$. 
\label{shd_0}
\end{proposition}
\begin{proof}
    Proof by contradiction. 
    $(\implies)$ Assume there is an edge $(u,v)\in \hat{G}^*$ and $(u,v) \notin G^*$. By the definition of $\hat{G}$, the existence of $P(u,v)$ in $\hat{G}$ means that $P(u,v)$ exists in $G$. Then $(u,v)\in G^*$ a contradiction.

     $(\impliedby)$ Assume there is an edge $(u,v)\in G^*$ and $(u,v) \notin \hat{G}^*$. Then $P(u,v)$ exists in $G$. By the definition of $\hat{G}$, all of the edges in $G$ are also in $\hat{G}$. Then $P(u,v)$ exists in $\hat{G}$ and $(u,v)\in \hat{G}^*$ a contradiction.   
\end{proof}
\par \hspace{2em}
The main method of the LLM-ISM approach is Algorithm \ref{alg:ssim} the iterative construction of $SSIM^*$. The input is a set of variables $X$, an integer $k$ specifying the number of variables to consider in a single LLM query, and a research question. The graph $\hat{G}$ is initialized as an empty digraph. For each variable $i$ under consideration, the LLM will predict the causal edges from $i$ to the target variable. When $k=1$, the algorithm operates in pairwise mode and the LLM expert is responsible for determining whether a single edge exists in $\hat{G}$. This results in $O(n^2)$ queries to the LLM expert. When $k=n$, the algorithm provides a variable $i$ to the LLM and asks the LLM expert to determine whether $n$ edges exist in $\hat{G}$ with a single inference request to the LLM. This results in $O(n)$ queries to the LLM expert.
\par \hspace{2em}
The LLM expert is provided the research question, a variable under consideration, and a list of variables for analysis. The LLM is instructed that it is performing a task for ISM and that it should assume the persona of an expert in qualitative systems analysis. The task to perform is to determine whether the provided variable influences any of the other variables in the variable list. The LLM should for each variable in the variables list, perform a pairwise comparison with the provided factor to determine the directional relationship between the two factors. Any of the other variables in the list should be used for context. Each variable in the variable list should have an entry in the output list. The LLM should respond with "V" if the provided variable directly or indirectly influences the output variable. The LLM should respond with "O" if there is no influence or insufficient evidence of a directional influence. The LLM is advised that false positive relationships are worse than missing weak relationships. The full prompt is provided in the Appendix.

\begin{algorithm}
\caption{Iteratively Construct $SSIM^*$}
\label{alg:ssim}
\begin{algorithmic}[1] 
\Procedure{constructSsimStar}{$X, k, research\_question$}
    \State $\hat{G} \gets \{\}$ 
    \For{$i \in X$}
        \State batchedFactors $\gets$ batch($X-\{x\}$, k)
        \For { factorList $\in$ batchedFactors}
            \State predictedEdges $\gets queryLlmExpert(research\_question, i, factorList)$ 
            \For{$(u,v) \in$ predictedEdges}
                \State $\hat{G}.add((u,v))$
            \EndFor
        \EndFor
    \EndFor
    \State $\hat{G}^* \gets transitiveClosure(\hat{G})$
    \State \Return $\hat{G}^*$
\EndProcedure
\end{algorithmic}
\end{algorithm}

\subsection{Incorporating SME Knowledge}

Rather than rely on the LLM's general knowledge, we extend the main approach in Algorithm \ref{alg:ssim_with_rag} to allow for in-context learning. An additional argument is added to allow passing documents to the algorithm. The documents are all converted to text using Docling \cite{docling}. The most relevant text chunks are retrieved using a query vector created from the research question and factors under consideration. The chunks are stored and passed in the LLM query.

\begin{algorithm}
\caption{Iteratively Construct $SSIM^*$ With RAG}
\label{alg:ssim_with_rag}
\begin{algorithmic}[1] 
\Procedure{constructSsimStarWithRag}{$X, k, documents, research\_question$}
    \State $\hat{G} \gets \{\}$ 
    \For{$i \in X$}
        \State batchedFactors $\gets$ batch($X-\{x\}$, k)
        \For { factorList $\in$ batchedFactors}
            \State $textChunks \gets []$
            \For {$j \in factorList$}
                \State $queryVector \gets embed(research question, i, j)$
                \State $textChunks.add(topTextChunks(queryVector, documents))$
            \EndFor
            \State predictedEdges $\gets queryLlmExpert(research\_question, i, factorList, textChunks)$ 
            \For{$(u,v) \in$ predictedEdges}
                \State $\hat{G}.add((u,v))$
            \EndFor
        \EndFor
    \EndFor
    \State $\hat{G}^* \gets transitiveClosure(\hat{G})$
    \State \Return $\hat{G}^*$
\EndProcedure
\end{algorithmic}
\end{algorithm}

\subsection{Experimental Setup}
This study evaluates the benchmark against three models: GPT-5.2, GPT-OSS 120B, and Nemotron Super 3 \cite{openaigpt5card, openai_oss120b, nvidia2025nvidianemotron3efficient}. The hardware setup is an enterprise server with 8 H100 Nvidia GPUs. The models are all served via the OpenAI Responses API. The models for GPT-OSS 120B and Nemotron 3 Super were full floating point 8 rather than the quantized version of the models. 

\subsection{Data Collection}
The benchmark dataset consists of 9 peer-reviewed publications that reported the results of the ISM process and did not rely solely on interviews or surveys \cite{I4T, smartOrganization, automationRobots, constructionWorkers, socialMediaBurnout, navigatingGenAiInHE, circularEconomy, fruitAndVegetable, truckDriverShortage}. The factors table and final reachability matrices were extracted and stored for evaluation. 

\section{Results}

This section presents the results of each methodology in discovering the ISM final reachability graph. The evaluation datasets are described by the number of nodes and edges. Since this study is inference only, we consider the dataset as test only and directly provide SHD and F1-score metrics. 

\subsection{Graph Statistics}
 A summary view of the 9 studies used in this evaluation are in Table \ref{graph_stats}. The 9 graphs are small with number of nodes ranging from 7 to 17 and edge density between 0.31 and 0.57. The fruitAndVegetables study is smallest with only 7 nodes and 17 edges. The i4t study is largest with 17 nodes and 115 edges. The average number of nodes is 11.7 nodes and average edges is 55.4. Graphs of this size are typical in ISM studies  \cite{smartISM}.  \\

 \begin{table}
\begin{tabular}{lrrr}
\toprule
Study & Nodes & Edges & Density \\
\midrule
                      i4t & 17 & 115 & .40  \\
       smartOrganizations & 10 & 57 & .57 \\
         automationRobots & 8 & 30 & .47 \\
      constructionWorkers & 13 & 55 & .33 \\
              socialMedia & 15 & 69 & .31 \\
   genAIInHigherEducation & 9 & 37 & .46 \\
          circularEconomy & 13 & 54 & .32 \\
       fruitAndVegetables & 7 & 17 & .35 \\
      truckDriverShortage & 13 & 65 & .38 \\
\bottomrule
\end{tabular}
 \centering
 \caption{ISM Graph Statistics}
 \label{graph_stats}
\end{table}

\subsection{Methodology Evaluation}

Table \ref{rag_stats} shows the results of running three models against the benchmark dataset. The best score for metric is in bold for each model. Given a specific value of $k$ and LLM model, the total SHD was calculated by summing the SHD from each study. The SHD for a study was calculated using $SHD(G^{*}, \hat{G}^{*})$ where $G$ is the published SSIM and $\hat{G}$ is the graph output from the LLM. The true positives, false positives, and false negatives were summated across all studies to calculate the F1-score. 
\par \hspace{2em}
GPT-5.2 outperforms the other models both in terms of F1-score and SHD. The highest F1-scores were from GPT-5.2 using rowwise comparison.
GPT-5.2 rowwise achieved F1-score of 0.77 using the RAG approach.
The lowest SHD scores were obtained with GPT-OSS-120B using the full graph approach.
GPT-OSS-120B full graph achieved 135 SHD using the RAG approach.
All models performed better as the number of comparisons in a single prompt increased.

\begin{table}
\setlength{\tabcolsep}{11pt}
\begin{tabular}{lrcrcrc}
\toprule
Batch Size & \multicolumn{2}{c}{GPT-5.2} &  \multicolumn{2}{c}{Nemotron 3} & \multicolumn{2}{c}{GPT-OSS-120B} \\
\hline
 & SHD $\downarrow$ & F1-Score $\uparrow$ & SHD $\downarrow$ & F1-Score $\uparrow$ & SHD $\downarrow$ & F1-Score $\uparrow$ \\
\midrule
 Pairwise (1) & 170 & 0.75 & 190 & 0.73 & 180 & 0.74 \\
 2            & 185 & 0.74 & 180 & 0.73 & 195 & 0.72 \\
 4            & 174 & 0.74 & \textbf{175} & 0.74 & 195 & 0.72 \\
 Rowwise (n)  & \textbf{160} & \textbf{0.77} & 180 & 0.73 & 165 & \textbf{0.75} \\
 Full Graph  & 164 & 0.73 & \textbf{175} & \textbf{0.75} & \textbf{135} & 0.73 \\
\bottomrule
\end{tabular}
 \centering
 \caption{Comparison of Causal Graph Methods}
 \label{rag_stats}
\end{table}

\section{Conclusion}

Despite advances in causal graph research using LLMs, LLMs have not been widely adopted in ISM processes. This paper proposes an integrated LLM-ISM approach and explores the suitability of incorporating LLMS into the ISM process. This approach shows that ISM is a suitable application for LLMs to aid in SSIM discovery. We benchmarked pairwise, k-wise, rowwise, and full graph prompting methods against a new ISM benchmark. SHD and F1-Score are presented for GPT-5.2, Nemotron 3 Super, and GPT-OSS-120B. GPT-5.2 outperformed both Nemotron 3 Super and GPT-OSS-120B. Across all models, the causal graph discovery methods for ISM perform best using rowwise and full graph methods.

\subsection{Future Work}
This methodology can scale well to studies with more variables than traditional ISM studies. A future work is to apply this methodology to Posttramautic Stress Disorder (PTSD) care. The modeling activity is to find the hierarchical relationships between barriers and enablers during treatment and recovery of individuals with PTSD. Rather than rolling variables up into high level concepts, the proposed methodology can examine all the variables at a detailed and granular level. 
\par \hspace{2em}
Another interesting topic for future investigation is the application of this integrated LLM-ISM technique to agentic workflows. Agents powered by LLMs will have the potential to use ISM graphs as in-context learning during planning or execution to make better and more deterministic decisions. Beyond using the ISM graphs for in-context learning, the graphs may also be incorporated into the training phase of the LLMs themselves to ensure causal inference is learned.

\section{Acknowledgments}

This research used resources of the Infrastructure for Science, AI, And Computing (ISAAC) AI-as-a-Service System at the Oak Ridge National Laboratory, which is supported by the Office of Science of the U.S. Department of Energy under Contract No. DE-AC05-00OR22725.

\bibliographystyle{vancouver}
\bibliography{main}

\end{document}